\let\oldvec\vec
\documentclass{llncs}
\let\vec\oldvec
\usepackage{prooftree} 
\usepackage{graphicx}
\usepackage{MnSymbol} 

\DeclareGraphicsExtensions{.eps,.jpg}

\newcommand\LCGp{LCG$_{\mathrm{p}}$}

    \newcommand\nat{\mathbf{N}}

\newcommand\seq\vdash

\newcommand\naur{\ \big|\ } 
\newcommand\backus{\ {:}{:}{=}\ } 

\usepackage{gb4e}

\newcommand\exun{
\begin{tabular}[t]{l} 
\textsc{Example 1} 
\\[1ex] 
\begin{tabular}{rl} 
$\left.\begin{array}{rr}
11& \comp{I}
\end{array}\right\}$
& Sophie 
\end{tabular} 
\\ \\   
\begin{tabular}{ll}
$\left.\begin{array}{rr}
11&\ppari{O}\\
&\otimes \\  
00&\qari{\sss^\perp}\\  
&\otimes \\ 
12&\pari{O}\\ 
&\comp{\otimes} \\  
13&\qqari{O}
\end{array}\right\}$ 
& gave 
\end{tabular} 
\\ \\ 
\begin{tabular}{rl}
$\left.\begin{array}{rr}
13&\pari{I}\\
&\comp{\otimes} \\  
14&\qari{O}\\  
\end{array}\right\}$ 
& a 
\end{tabular} 
\\ \\ 
\begin{tabular}{ll} 
$\left.\begin{array}{rr}
14& \comp{I}
\end{array}\right\}$
& kiss 
\end{tabular} 
\\ \\   
\begin{tabular}{rl}
$\left.\begin{array}{rr}
12&\pari{I}\\
&\comp{\otimes} \\  
15&\qari{O}\\  
\end{array}\right\}$ 
& to 
\end{tabular} 
\\ \\ 
\begin{tabular}{rl} 
$\left.\begin{array}{rr}
15& \comp{I}
\end{array}\right\}$
& Christian 
\end{tabular} 
\\ \\ 
\begin{tabular}{rl}
$\left.\begin{array}{rr}
00& \comp{\sss}   \\ 
\end{array}\right\}$ 
& \textit{(sentence)}   
\end{tabular} 
\end{tabular}}

\newcommand{\exdeux}{ 
\scalebox{0.85}{
\begin{tabular}[t]{l} 
\textsc{Example 2} 
\\[1ex]  
\begin{tabular}{rl} 
$\left.\begin{array}{rr}
21& \comp{I}
\end{array}\right\}$
& Christian  
\end{tabular} 
\\ \\   
\begin{tabular}{ll}
$\left.\begin{array}{rr}
21&\ppari{O}\\
&\otimes \\  
00&\qari{\sss^\perp}\\  
&\comp{\otimes} \\ 
22 & \qari{O} 
\end{array}\right\}$ 
& gave 
\end{tabular} 
\\ \\   
\begin{tabular}{rl}
$\left.\begin{array}{rr}
23&\pari{I}\\
&\comp{\otimes} \\  
24&\qari{O}\\  
\end{array}\right\}$ 
& a 
\end{tabular} 
\\ \\ 
\begin{tabular}{rl} 
$\left.\begin{array}{rr}
24& \comp{I}
\end{array}\right\}$
& book  
\end{tabular} 
\\ \\ 
\begin{tabular}{rl}
$\left.\begin{array}{rr}
25&\pari{I}\\
&\comp{\otimes} \\  
26&\qari{O}\\  
\end{array}\right\}$ 
& to 
\end{tabular} 
\\ \\   
\begin{tabular}{rl} 
$\left.\begin{array}{rr}
26& \comp{I}
\end{array}\right\}$
& Anne  
\end{tabular} 
\\ \\    
\begin{tabular}{rl}
$\left.\begin{array}{rr}
25 &\ppari{\pari{O}}\\ 
&\otimes\\ 
23&\qari{O}\\ 
&\comp{\otimes}\\ 
22 &\qari{I}\\ 
&\otimes\\ 
27&\pari{O}\\ 
&\otimes\\ 
28&\qqari{O} 
\end{array}\right\}$
& and   
\end{tabular} 
\\ \\   
\begin{tabular}{rl}
$\left.\begin{array}{rr}
28&\pari{I}\\
&\comp{\otimes} \\  
29&\qari{O}\\  
\end{array}\right\}$ 
& a 
\end{tabular} 
\\ \\   
\begin{tabular}{rl} 
$\left.\begin{array}{rr}
29& \comp{I}
\end{array}\right\}$
& kiss 
\end{tabular} 
\\ \\   
\begin{tabular}{rl}
$\left.\begin{array}{rr}
27&\pari{I}\\
&\comp{\otimes} \\  
20&\qari{O}\\  
\end{array}\right\}$ 
& to 
\end{tabular} 
\\ \\   
\begin{tabular}{rl} 
$\left.\begin{array}{rr}
20& \comp{I}
\end{array}\right\}$
& Sophie  
\end{tabular} 
\\ \\   
\begin{tabular}{rl}
$\left.\begin{array}{rr}
00& \comp{\sss}   \\ 
\end{array}\right\}$ 
& \textit{(sentence)}   
\end{tabular} 
\end{tabular}}}

\newcommand{\extrois}{
\begin{tabular}[t]{l} 
\textsc{Example 3}\\[1ex]  
\begin{tabular}{rl} 
$\left.\begin{array}{rr}
31& \comp{I}
\end{array}\right\}$
& Sophie 
\end{tabular} 
\\ \\   
\begin{tabular}{ll}
$\left.\begin{array}{rr}
31&\ppari{O}\\
&\otimes \\  
00&\qari{\sss^\perp}\\  
&\comp{\otimes} \\ 
32&\qari{O}\\
\end{array}\right\}$ 
& liked 
\end{tabular} 
\\ \\   
\begin{tabular}{rl}
$\left.\begin{array}{rr}
32&\pari{I}\\
&\comp{\otimes} \\  
33&\qari{O}\\  
\end{array}\right\}$ 
& a 
\end{tabular} 
\\ \\   
\begin{tabular}{rl} 
$\left.\begin{array}{rr}
34& \comp{I}
\end{array}\right\}$
& book 
\end{tabular} 
\\ \\   
\begin{tabular}{rl}
$\left.\begin{array}{rr}
34&\ppari{O}\\
&\comp{\otimes}\\
33&\qari{I}\\ 
&\otimes\\  
35 &\pari{I}\\ 
&\pa\\
36&\qqari{O}\\ 
\end{array}\right\}$ 
& that 
\end{tabular} 
\\ \\   
\begin{tabular}{rl} 
$\left.\begin{array}{rr}
37& \comp{I}
\end{array}\right\}$
& Christian 
\end{tabular} 
\\ \\   
\begin{tabular}{ll}
$\left.\begin{array}{rr}
37&\ppari{O}\\
&\otimes \\  
36&\qari{I}\\  
&\comp{\otimes} \\ 
35&\qari{O}\\
\end{array}\right\}$ 
& liked 
\end{tabular} 
\\ \\   
\begin{tabular}{rl}
$\left.\begin{array}{rr}
00& \comp{\sss}   \\ 
\end{array}\right\}$ 
& \textit{(sentence)}   
\end{tabular} 
\end{tabular} }

\newcommand\calP{\mathsf{P}}

\newcommand\calTL{\mathsf{Lp}}

\newcommand\calTMLL{\mathsf{L}}

\newcommand\FI{{\calTMLL^\bullet}}
\newcommand\FO{{\calTMLL^\circ}}
\newcommand\FIH{{\calTMLL_h^\bullet}}
\newcommand\FOH{{\calTMLL_h^\circ}}
\newcommand\FL{{\calTMLL}}

\newcommand\OC{\circ}
\newcommand\IC{\bullet}

\newcommand\pt\otimes 
\newcommand\lang{\mathcal{L}}
\newcommand\calV{\mathcal{V}}
\newcommand\calU{\mathcal{U}}
\newcommand\calC{\mathcal{C}}
\newcommand\calCp{\mathcal{C}_{\otimes}}
\newcommand\calL\lang
\newcommand\lex{\mathrm{lex}}

\newcommand\vide{\emptyset} 
\newcommand\llts{\mathbin{\otimes}}

\newcommand\strictsubset{\varsubsetneq} 

\newcommand\ma[1]{``\emph{#1}''}

\newcommand\SPF{\ensuremath{\textsf{sPF}}} 

\newcommand\pari[1]{\overgroup{\compar{#1}}}
\newcommand\qari[1]{\undergroup{\compar{#1}}}
\newcommand\ppari[1]{\overgroup{\ \pari{#1}}} 
\newcommand\qqari[1]{\undergroup{\ \qari{#1}}}
\newlength{\toto}
\newlength{\tyty}
\newcommand\compl[2]{\settowidth{\toto}{$#1$}\setlength{\tyty}{#2}\addtolength{\tyty}{-\toto}\hspace*{\tyty}#1} 
\newcommand\comp[1]{\compl{#1}{2em}} 
\newcommand\compar[1]{\compl{#1}{1em}} 

\newcommand\calG{\mathcal{G}} 

\newcommand\yields\longrightarrow 
\newcommand\sss{\mathsf{s}}
\newcommand\fails\bot
\newcommand\lto{\mathbin{\backslash}}
\newcommand\lfrom{\mathbin{/}}

\newcommand\pa{\mathbin{\wp}}
\newcommand\lts{\mathbin{\otimes}}

\author{Roberto Bonato\inst{1}\thanks{I am deeply indebted to my co-author for having taken up again after so many years our early work on learnability for $k$-valued Lambek grammars, extended and coherently integrated it into the framework of learnability from proof frames.} \and Christian Retor\'e \inst{2}\thanks{Thanks to CNRS and to IRIT  for my sabbatical year, to the Loci ANR project for its intellectual and financial support, to C. Casadio, M. Moortgat for their encouragement  and to A. Foret and to the anonymous reviewers for their helpful remarks.}}
\institute{Questel SAS, Sophia Antipolis, France \and IRIT, Toulouse, France \& Univ. Bordeaux, France}

\title{Learning Lambek grammars  from proof frames}

\begin{document} 

\maketitle

\abstract{In addition to their limpid interface with semantics, 
categorial grammars 
enjoy another important property: learnability. 
This was first noticed by Buskowsky and Penn and further studied by Kanazawa, for Bar-Hillel categorial grammars. 

What about Lambek categorial grammars? 
In a previous paper we showed that product free Lambek grammars 
where learnable from structured sentences, the structures being incomplete natural deductions. 
These grammars were shown to be unlearnable from strings by Foret and Le Nir. 
In the present paper we show that Lambek grammars, possibly with product, are learnable from proof frames that are incomplete proof nets.

After a short reminder on grammatical inference \`a la Gold, 
we provide an algorithm that learns Lambek grammars with product from proof frames and we prove 
its convergence. We do so for $1$-valued also known as rigid Lambek grammars with product, since standard techniques can extend our result to $k$-valued grammars. Because of the correspondence between cut-free proof nets and normal natural deductions, 
our initial result on product free Lambek grammars can be recovered. 

We are sad to dedicate the present paper to Philippe Darondeau, with whom we started to study such questions in Rennes at the beginning of the millennium, and who passed away prematurely. 

We are glad to dedicate the present paper to Jim Lambek for his 90 birthday: he is the living proof  that research is an eternal  learning process.  
}

\section{Presentation} 

Generative grammar exhibited two characteristic properties of the syntax of 
human languages that distinguish them from other formal languages:
\begin{enumerate}
\item Sentences should be easily parsed and generated, since we speak and understand each other in real time. 
\item Any human language should be easily learnable,  preferably from not so many positive examples, as the study of first language acquisition shows. 
\end{enumerate}

\emph{Formally}, the first point did receive a lot of attention, leading to the class of mildly context sensitive languages \cite{JVW91}: 
they enjoy polynomial parsing but are rich enough to describe natural language syntax. A formal 
account of learnability was more difficult to find. Furthermore, as soon as a notion of formal learnability was proposed, the first results seemed so 
negative that the learnability criterion was left out of the design of syntactical formalisms. 
This negative result can be stated 
as follows:  any class of languages that contains all the regular languages cannot be learnt. 

It should be explained why this result was considered so negatively. 
By that time,
languages were viewed through the Chomsky hierarchy (see figure \ref{hierarchy}).  
Given that regular languages are the simplest class and 
that human languages were known to go beyond regular languages, 
it seemed that there could not exist an algorithm that learns
a large class as the one of human languages. 
This pessimistic viewpoint was erroneous for at least two reasons: 
\begin{itemize} 
\item  The class of human languages does not include all regular languages and it is likely that it does  not  even include a single regular language. The figure \ref{hierarchy} gives the present hypothesis on the class of human languages. 
\item The positive examples were thought to be sequences of words, while it has been shown long ago that grammatical rules operate on structured sentences and phrases (that are rather trees or graphs), see e.g. \cite{Berwick2011} for a recent account. 
\end{itemize} 

\begin{figure}[t]  
\label{hierarchy}
\includegraphics[scale=0.4]{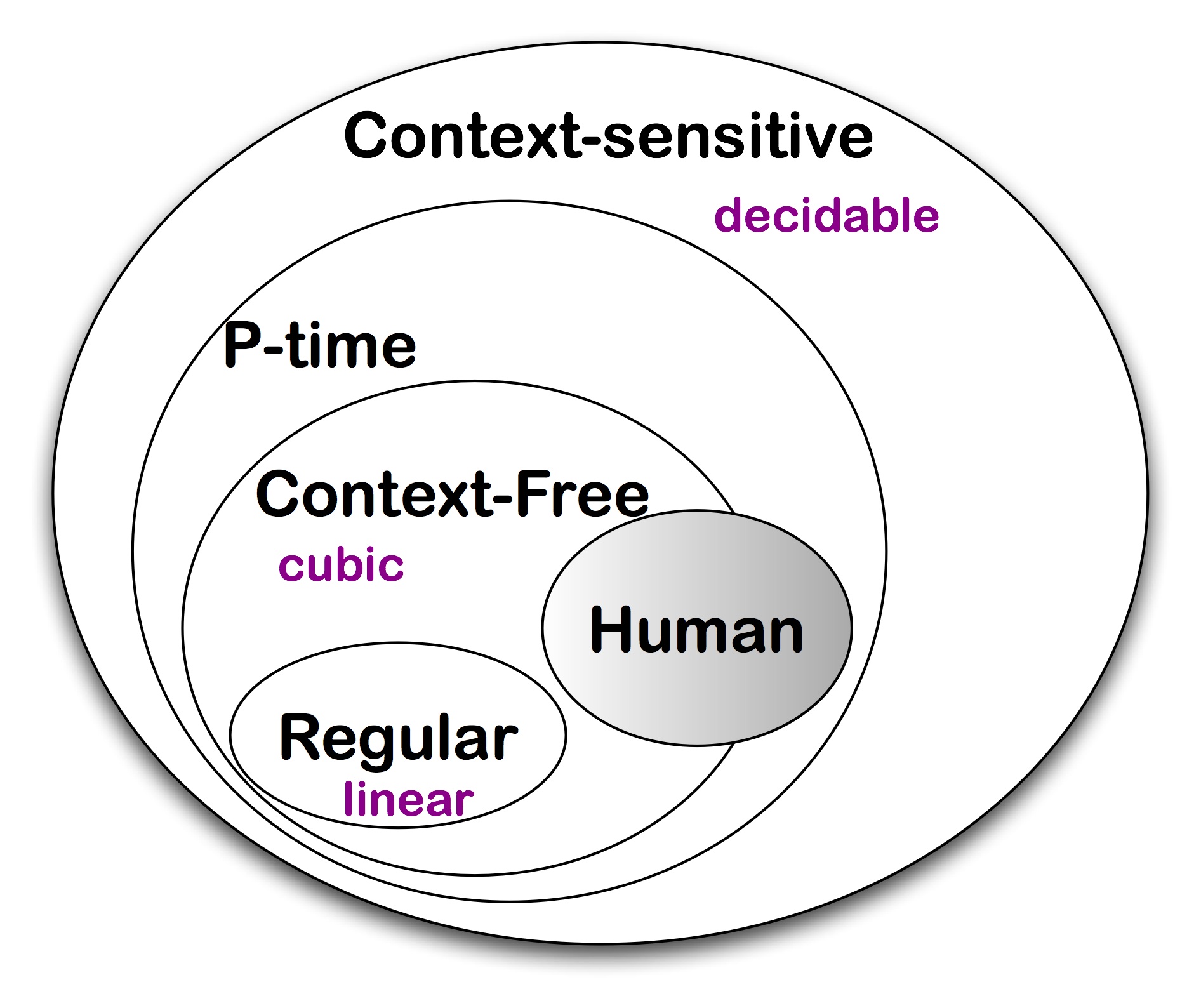}
\caption{Human languages and the classes of the Chomsky hierarchy (with parsing complexity).}
\end{figure} 

Although we shall recall it more precisely in the first section of the present paper, let us make some comments on 
Gold's notion of \emph{learning a class of languages} generated by a class of grammars  $\calG$. 
According to Gold, a learning function $\phi$ maps a sequence of sentences $e_1,\ldots,e_n$ to a grammar $G_n=\phi(e_1,\ldots,e_n)$ in the class 
$\calG$  in such a way that, when the examples $(e_i)_{i\in \nat}$ 
enumerate a  language $\calL(G)$ of a grammar $G$ in  $\calG$ i.e. 
$\calL(G)=\{e_i\ |\ i\in \nat\}$
there exists an integer $N$ such that for all $n>N$ the grammars  $G_n$ are constantly equal to $G_N$  which generates  the same language i.e. 
$\calL(G_n)=\calL(G)=\{e_i\ |\ i\in \nat\}$. The fundamental point is that the function learns a \emph{class} of languages: the algorithm eventually finds out that the enumerated language  cannot be any other language in the class. This means that 
the very same language can be learnable as a member of a learnable class of languages, and unlearnable as the member of another class of languages. Although surprising at first sight, this notion according to which one \emph{learns  in a predefined class of languages}  is rather compatible 
with our present knowledge of first language acquisition.

Overtaking the pessimistic view of Gold's theorem,  Angluin established in the 80s 
that some large but transversal classes of languages were learnable in Gold's sense. \cite{Angluin1980pattern}
Regarding categorial grammars, Buszkowski and Penn defined in late 80s \cite{BP90,Bus87b}
an algorithm that learns basic categorical grammars from structured sentences, that are called functor-argument structures, 
and Kanazawa proved in 1994 that their algorithm converges: it actually learns categorial grammar in Gold's sense. 
\cite{Kan98,Kan94}

The result in the present paper is much in the same vein as the ones by Buskowski, Penn and Kanazawa.  
\begin{description} 
\item[Section \ref{gold}] 
We first recall the Gold learning paradigm, identification in the limit from positive examples. 
\item[Sections \ref{categorial}, \ref{cgpf}] 
Next we briefly present Lambek categorial grammars with product, and  define their parsing as the construction of cut-free proof nets. We also introduce the structures we shall learn from, 
that we call \emph{proof frames}.
Indeed, Lambek grammars (with or without product) ought to be learnt from \emph{structured sentences} since Foret and Le Nir established that they cannot be learnt from strings \cite{FL02coling}. 
Informally, proof-frames are name-free parse-structures, i.e. name-free proof-nets,  just like functor-argument structures  are name-free natural deduction used to learn basic categorial grammars. 
\item [Sections \ref{uni},\ref{RG},\ref{conv}]
After a reminder on unification in relation to  categorial grammars, we present our algorithm that learns rigid Lambek categorial grammars with product from proof frames. We illustrate it on examples involving introduction rules (that are not part of basic categorial grammars) and product rules (that are not part  of product free Lambek grammars). We then prove the convergence of this algorithm. 
\item[Section \ref{nd}]
We there show that the present result strictly encompasses our initial result \cite{BR01lll} that learns rigid product-free Lambek grammars 
from name-free natural deductions. To do so, we give the bijective correspondence between \emph{cut-free} proof nets for the product-free Lambek calculus and \emph{normal} natural deduction that are commonly used as parse structures. 
\item[In the conclusion,] we discuss the merits and limits of the present work. 
We briefly explain how it can generalise to $k$-valued Lambek grammars with product and suggest direction for obtaining corpora with proof frame annotations from dependency-annotated corpora. 
\end{description}


\section{Exact learning \`a la Gold: a brief reminder} 
\label{gold} 

We shall just give a brief overview of the Gold learning model of \cite{Gold67}, 
with some comments, and explain why his 
famous unlearnability theorem of \cite{Gold67} (theorem \ref{unlearnable} below) is not as negative as it first seemed  --- as the result of \cite{Angluin1980pattern} and of the present article shows.

The principles of first language acquisition as advocated by Chomsky \cite{CP75} and more recently by Pinker
\cite{Pin95,Pin95rabbits}
can be very roughly summarised as follows: 

\begin{enumerate} 
\item \label{positive} 
One \emph{learns from positive examples only}: an argument says that in certain civilisations 
children uttering ungrammatical sentences are never corrected although they learn the grammar 
just as fast as ours 
--- this can be discussed, since the absence of reaction might be considered as negative evidence, 
as well as the absence of some sentences in the input.  
\item \label{restriction} 
The target language is reached by \emph{specialisation},  i.e.  by restricting word order from tentative languages with a freer word order:
rare are the learning algorithms for natural language that proceed by specialisation although, when starting from semantics,  there are such algorithms as  the one by Tellier \cite{Tellier2008icgi} 
\item \label{root} 
\emph{Root meaning} is learnt first, and as part of this root meaning, the argumental structure (also known as  valencies in dependency grammars) are known before the grammar learning process actually starts. 
This implies that in the learner's utterances exactly all needed words are present, possibly in a non correct order. This  enforces the idea that one learns by specialisation --- the afore mentioned work by Tellier actually uses  argument structures as inputs \cite{Tellier2008icgi} 
\item \label{POS} 
The examples that the child is exposed to are not so numerous: this is known as the \emph{Poverty Of Stimulus} argument.
It has been widely discussed since 2000 in particular for supporting  quantitative methods. \cite{CP75,PullumScholz2002,RealiChristiansen2005cogscience,Berwick2011}
\end{enumerate} 


In his seminal 1967 paper, Gold introduced a formalisation of the process of the acquisition of one's first language grammar. It strictly follows the first principle stated above, which is the easiest to formalise: the formal question he addressed could be more generally stated as \emph{grammatical inference from positive examples}. 
It also should be said that Gold's notion of learning may be used for other purposes, every time  one wants to extract some regularity out of sequences  observations. It has been applied to genomics (what would be a grammar generating the strings issued from DNA sequences) 
and diagnosis (what are the regular behaviours of the system, what would be a grammar generating the sequences of normal observations provided by captors for detecting abnormal behaviours). 

We shall provide only a minimum of information on formal languages and grammars. Let us just say that a  language is a subset of an inductive class  $\calU$. Elements of $\calU$ usually are finite sequences (a.k.a. strings) of words, or trees whose leaves are labelled by words, or graphs whose vertices are words. \footnote{We here say ``words" because they are linguistic words, while other say ``letters" or ``terminals," and we say ``sentences" for sequences of words where others say ``words" for sequences of ``letters" or ``terminals"} 
A grammar $G$ is a finitely described process generating the objects of a language $\calL(G)\subset\calU$.  
The membership question is said to be decidable for a grammar $G$ when the
characteristic function of $L(G)$ in $\calU$ is computable. 
The most standard example of $\calU$ is $\Sigma^*$ the set of finite sequences over some set of symbols (e.g. words)  $\Sigma$.
The phrase structure grammars of Chomsky-Schutzenberger are the most famous grammars producing languages that are parts of $\Sigma^*$. Lambek categorial grammars and basic categorial grammars are an alternative way to generate sentences as elements of  
$\Sigma^*$: they produce the same languages as context-free languages 
\cite[chapters 2]{BGS63,Pen93,MootRetore2012lcg}. Finite labeled trees also are a possible class of objects. For instance a regular tree grammar produces a tree language, whose yields define a context free string language. In the formal study of  human languages, $\calU$ usually consists in strings of words or in  trees with words on their leaves. 

\begin{definition}[Gold, 1967, \cite{Gold67}] \label{identification} 
A \emph{learning function} for a class of grammars $\calG$ producing $\calU$-objects  ($\lang(G)\subset\calU$) is a partial function $\phi$ that maps any finite sequence of positive examples $ex_1,ex_2,\ldots,ex_k$  with  $ex_i\in\calU$ to a grammar $G_i=\phi(ex_1,ex_2,\ldots,ex_k)$ of the class of grammars $\calG$  such that:
\begin{description} 
\item[if] $(e_i)_{i\in I}$ is any enumeration of a language $\lang(G)\subset\calU$ with $G\in\calG$,   
\item[then] there exists an integer $N$ such that:
\begin{itemize}
\item $G_P=G_N$ for all $P\geq N$. 
\item $\lang(G_N)=\lang(G)$. 
\end{itemize}
\end{description} 
\label{goldcvg} 
\end{definition}

Several interesting properties of learning functions have been considered:

\begin{definition}\label{learningprop} 
A learning function $\phi$ is said to be 
\begin{itemize} 
\item 
\emph{effective} or computable when $\phi$ is recursive. 
In this case one often speaks of a \emph{learning algorithm}. In this article, we shall only consider effective learning functions:
this is consistent both with  language being viewed as a computational process and with applications to computational linguistics. 
Observe that the learning function does not have to be  a total recursive function: it may well be undefined for some sequences of sentences and still be a learning function. 
\item 
\emph{conservative} if
$\phi(e_1,\ldots,e_p,e_{p+1})=\phi(e_1,\ldots,e_p)$ whenever $e_{p+1}\in\calL(\phi(e_1,\ldots,e_p))$. 
\item 
\emph{consistent} if $\{e_1,\ldots,e_p\}\subset\calL(\phi(e_1,\ldots,e_p))$  whenever $\phi(e_1,\ldots,e_p)$ is defined. 
\item 
\emph{set driven} if  $\phi(e_1,\ldots,e_p)=\phi(e'_1,\ldots,e'_q)$ whenever $\{e_1,\ldots,e_p\} = \{e'_1,\ldots,e'_q\}$
--- neither the order of the examples nor their repetitions matters. 
\item 
\emph{incremental} if there exists a binary function $\Psi$ such that\\ $\phi(e_1,\ldots,e_p,e_{p+1})=\Psi(\phi(e_1,\ldots,e_p),e_{p+1})$
\item 
\emph{responsive} 
if the image $\phi(e_1,\ldots,e_p)$ is defined  whenever  there exists $L$ in the class with $\{e_1,\ldots,e_p\}\subset L$ 
\item 
\emph{monotone increasing}  when $\phi(e_1,\ldots,e_p,e_{p+1})\subset\phi(e_1,\ldots,e_p)$
\end{itemize} 
\end{definition} 

The algorithm for learning Lambek grammars that we propose in this paper enjoys all those properties. All of them seem to be sensible with respect to first language acquisition but the last one: 
indeed, as said above,  children rather learn by specialisation. 

It should also be observed that the learning algorithm applies to a \emph{class of languages}. 
So it is fairly possible that a given language $L$ which both belongs to the classes $\calG_1$ 
and $\calG_2$ can be identified as a member of $\calG_1$ and not as a member of $\calG_2$. Learning $L$ in such a setting is nothing more than to be sure, given the examples seen so far, that the language is not any other language in the class. 

The classical result from the same 1967 paper by Gold \cite{Gold67} that has be over interpreted see e.g. \cite{Angluin1980,Johnson2004gold} can be stated as follows:

\begin{theorem}[Gold, 1967, \cite{Gold67}] \label{unlearnable}
If a class $\calG_r$ of grammars generates 
\begin{itemize} 
\item   languages $(L_i)_i\in\nat$ with $L_i\in\nat$ 
which are strictly embedded that is $L_i\strictsubset L_{i+1}$ for all $i\in\nat$ 
\item together with the union of all these languages $\cup_{i\in\nat} L_i\in \calG_r$ 
\end{itemize} 
then no function may learn $\calG_r$. 
\end{theorem} 

\begin{proof} From the definition, we see that 
a learning function should have guessed the grammar of a language $\lang(G)$ with $G\in\calG$ after a finite number of examples in the enumeration of $\lang(G)$. 
Consequently, for any enumeration of any language in the class, 
\begin{exe}
\ex\label{changes} 
the learning function may only change its mind finitely many times. 
\end{exe} 
Assume that is a learning function $\phi$ for the class $\calG_r$.  Since the $L_i$ are nested as stated, 
we can provide an enumeration of $L=\cup L_i$ 
according to which
we firstly see examples $x_0^1,\cdots,x_0^{p_0}$ from $L_0$ until $\phi$ proposes  $G_0$ with $\calL{G_0}=L_0$,  
then we  see examples $x_1^1,\cdots,x_1^p$ in $L_1$ until $\phi$ proposes  $G_1$ with $\calL{G_1}=L_1$, 
then we  see examples $x_2^1,\cdots,x_2^p$ in $L_2$ until $\phi$ proposes  $G_2$ with $\calL{G_2}=L_2$, etc. 
In such an enumeration of $L$ the learning function changes its mind infinitely many times, conflicting with (\ref{changes}). Thus there cannot exists a learning function for the class $\calG_r$. 
\end{proof} 

Gold's theorem above has an easy  consequence that was interpreted quite negatively: 

\begin{corollary} 
No class containing the regular languages can be learnt.
\end{corollary}

Indeed, by that time the Chomsky hierarchy was so present that no one thought that transverse classes could be of any interest, let alone that they could be learnable. Nowadays, it is assumed that the syntax of human languages contains no regular languages and goes a bit beyond context
free languages --- as can be seen in figure \ref{hierarchy}.  It does not seem likely that human languages contain a series  of strictly embedded 
languages as well as their unions. Hence Gold's theorem does not prevent large and interesting classes of languages, like human languages, from being learnt. For instance Angluin showed that pattern languages, a transversal class can be learnt by identification in the limit \cite{Angluin1980pattern} and she also provided a criterion for learnability base on telltale sets: 

\begin{theorem}[Angluin, 1980, \cite{Angluin1980}] 
An enumerable family of languages $L_i$ with a decidable membership problem 
is effectively learnable whenever for each $i$ there is a computable finite $T_i\subset_f L_i$ such that  if $T_i\subset L_j$ then there exists $w\in (L_j\setminus L_i)$.     
\end{theorem}

As a proof that some interesting classes are learnable, 
we shall define particular grammars, Lambek categorial grammars with product,  and their associated structure languages,
before  proving that they can be learnt from these structures, named proof frames.

\section{Categorial grammars and the \LCGp\ class} 
\label{categorial}

Given a finite set of words $\Sigma$ and an inductively defined set of categories $\calC$ including a special category $\sss$ 
and an inductively defined  set of derivable sequents $\seq\ \subset\ (\calC^* \times \calC)$ (each of them being written $t_1,\ldots,t_n\seq t$) 
a categorial grammar $G$ is defined as map $\lex_G$ from words to finite sets of categories. An important property, as far as learnability 
is concerned,  is the maximal number of categories per word  i.e. $\max_{w\in\Sigma} |\lex_G(w)|$. When it is less than $k$, the categorial grammar $G$ 
is said to be \emph{$k$-valued}  and $1$-valued categorial grammars are said to be  \emph{rigid}. 

Some standard family of categorial grammars are:
\begin{enumerate} 
\item 
\emph{Basic categorial grammars BCG} also known as  AB grammars have their categories in 
$\calC \backus  \sss \naur  B \naur  \calC\lto \calC \naur  \calC\lfrom \calC$  and the derivable sequents are the ones that are derivable in the Lambek calculus 
with elimination rules only  $\Delta\seq A$ and $\Gamma\seq B\lfrom A$ (respectively $\Gamma\seq A\lto B$) yields  
$\Gamma,\Delta\seq B$ (respectively $\Delta,\Gamma\seq B$) 
--- in such a setting the empty sequence is naturally prohibited even without saying so. \cite{BH53} 
\item 
The original \emph{Lambek grammars}  \cite{Lam58}  also have their categories in the same inductive set 
$\calC \backus  \sss \naur  B \naur  \calC\lto \calC \naur  \calC\lfrom \calC$  and the derivable sequents are the ones that are derivable in the Lambek calculus 
without empty antecedent, i.e. with rules of  figure \ref{sequents}  except $\otimes_i$ and $\otimes_h$ --- a variant allows empty antecedents.  
\item 
\emph{Lambek grammars with product} (\LCGp)  have their categories in 
$\calCp \backus  \sss \naur  B \naur  \calCp\lto \calCp \naur  \calCp\lfrom \calCp  \naur  \calCp\otimes\calCp$  and the derivable sequents are the ones that are derivable in the Lambek calculus with product 
without empty antecedents with all the rules of figure \ref{sequents} --- a variant allows empty antecedents.  
\end{enumerate}

A phrase, that is a sequence of words $w_1 \cdots w_n$, is said to be of category $C$ according to  $G$ when, for every $i$ between $1$ and $p$ there exists $t_i\in\lex_G(w_i)$ such that $t_1,\ldots,t_n\seq C$ is a derivable sequent. When $C$ is $\sss$ the phrase is said to be a \emph{sentence} according to $G$. The string language generated by a categorial grammar is the subset of $\Sigma^*$ consisting in strings that are of category $\sss$ i.e. sentences. 
Any language generated by a grammar in one of the aforementioned classes of categorial grammars is context free. 


In this paper we focus on Lambek grammars with product (\LCGp). The explicit use of product categories in Lambek grammars is not so common. Indeed, one rather sees negative products that can be treated with categorial implication $\lto$ and $\lfrom$. 
For instance a category like $(a\otimes b)\lto c$ can be viewed as $b\lto (a\lto c)$ so they do not really involve a product. 
The comma in the left-hand side of the sequent, 
corresponding to the blanks between words are implicit products, but grammar and parsing of those can be defined without explicitly using the product. 
Nevertheless, there are cases  when the product is appreciated: 

\begin{itemize}
\item 
For analysing the French Treebank, Moot in \cite{moot10semi} assigns the category $((np\otimes pp)\lto (np\otimes pp))\lfrom(np\otimes pp)$  
to \ma{et} (\ma{and})  for sentences like: 

\begin{exe}
\ex \label{exetmoot} 
Jean donne un livre \`a Marie et une fleur \`a Anne.
\end{exe} 

\item 
According to Glyn Morrill \cite{Mor98,Morrill2011cg} past participles like \ma{raced} should be assigned the category 
$ ((CN\lto CN)/(N\lto  (N\lto \sss{-}))\otimes (N\lto (N\lto \sss{-}))$ where $\sss{-}$ is an untensed sentence 
in  sentences like: 
\begin{exe}\label{expastpartmorrill}
\ex 
The horse raced past the barn fell.  
\end{exe} 
\end{itemize} 

\begin{figure}[t] 
\label{sequents} 
The derivable sequents of the Lambek syntactic calculus with product are obtained from the axiom $C\seq C$ for any category $C$   and the rules are given below, where $A,B$ are categories and $\Gamma,\Delta$ finite sequences of categories: 
\begin{center}
\begin{tabular}{ccc} 
\begin{prooftree} 
\Gamma,B,\Gamma'\seq C
\quad 
\Delta\seq A
\justifies 
\Gamma,\Delta,A\lto B,\Gamma'\seq C
\using \lto_{h}
\end{prooftree} 
&\quad\quad\quad\quad&
\begin{prooftree} 
A,\Gamma\seq C
\justifies 
\Gamma \seq A\lto C
\using \lto_i\quad \Gamma\neq\vide
\end{prooftree} 
\\ &&  \\   &&  \\ 
\begin{prooftree} 
\Gamma,B,\Gamma'\seq C
\quad 
\Delta\seq A
\justifies 
\Gamma,B\lfrom A,\Delta,\Gamma'\seq C
\using \lfrom_{h}
\end{prooftree} 
&&
\begin{prooftree} 
\Gamma,A \seq C
\justifies 
\Gamma \seq  C\lfrom A
\using \lfrom_i\quad \Gamma\neq\vide
\end{prooftree} 
\\ &&  \\   &&   \\  
\begin{prooftree} 
\Gamma,A,B,\Gamma' \seq C
\justifies 
\Gamma,A\pt B,\Gamma' \seq C
\using \pt_{h}
\end{prooftree} 
&&
\begin{prooftree} 
\Delta \seq A
\quad 
\Gamma \seq B
\justifies 
\Delta,\Gamma\seq A\pt B
\using \pt_i
\end{prooftree} 
 \\ && 
\end{tabular} 
\end{center}
\caption{Sequent calculus rule for the Lambek calculus}
\end{figure}

\section{Categorial grammars generating proof frames} 
\label{cgpf} 

The classes of languages that we wish to learn include some proper context free languages \cite{BGS63}, hence they might be difficult to learn. 
So we shall learn them from \emph{structured sentences}, and this section introduces \emph{proof frames}  that we shall use as structured sentences.

A neat  natural deduction system for Lambek calculus with product is rather intricate \cite{AR07cie,Amblard07phd}, 
mainly because the product elimination rules have to be carefully commuted for having a unique normal form. 
Cut-free sequent calculus proofs neither are fully satisfactory because they are quite redundant and some of their rules can  be swapped. 
As explained in \cite[chapter 6]{MootRetore2012lcg} proof nets provide perfect parse structures for Lambek grammars even when they use the product --- when the product is not used, cut-free proof nets and normal natural deduction are isomorphic, as we shall show in subsection \ref{pn&nd}. 
Consequently the structures that we used for learning will be proof frames that are proof nets with missing informations. 
Let us see how categorial grammars generate such structures, and first let us recall the correspondence between polarised formulae of linear logic and Lambek categories.  

\subsection{Polarised linear formulae and Lambek categories} 
\label{formulae} 

A Lambek grammar is better described with the usual Lambek categories, while proof nets are better described with linear logic formulae.
Hence we need to recall the correspondence between these two languages as done in \cite[chapter 6]{MootRetore2012lcg}. Lambek categories (with product) are 
$\calCp$ 
defined in the previous section \ref{categorial}.  Linear formula $\FL$ are defined by: 

$$
\begin{array}{lcccccccc}
\FL & \backus  &\calP & \naur  & \calP^\perp & \naur  & (\FL \llts \FL) & \naur  & (\FL\pa\FL) \\
\end{array} 
$$ 

The negation of linear logic $\_)^\perp$ is only used on propositional variables from $P$ as the De Morgan laws allow: 

$\hfill (A^\perp)^\perp\equiv A\hfill \hfill (A\pa B)^\perp\equiv (B^\perp\otimes A^\perp)\hfill \hfill (A\otimes B)^\perp\equiv (B^\perp\pa A^\perp)\hfill $ 

 To translate Lambek categories into linear logic formulae, one has to distinguish the polarised formulae, which are either  outputs or positive formulae of $\FO$ or  inputs or negative formulae of $\FI$ 
 with $F\in\FO \iff F^\perp\in\FI$ and $(\FO\cup\FI)\strictsubset\FL$: 
 
$$\left\{
\begin{array}{lcccccccc}
\FO & \backus  &\calP & \naur  & (\FO \llts \FO) & \naur  & (\FI\pa\FO) & \naur  & (\FO\pa \FI)\\
\FI & \backus  & \calP^\perp & \naur  & (\FI \pa \FI) & \naur  & (\FO\llts\FI) & \naur  & 
(\FI\llts \FO) \\
\end{array}\right.
$$ 

 Any formulae of the Lambek $L$ calculus can be translated as an output formula $+L$ of multiplicative linear logic while its negation can be translated as an input linear logic formulae $-L$: 

$$
\begin{array}{|c||c|c|c|c|}
\hline
  L & \alpha\in\calP &  L=M\lts N &L=M\lto N &  L=N\lfrom M \\ \hline \hline
  +L & \alpha &  +M\llts +N & -M\pa +N & +N\pa -M \\ \hline
  -L & \alpha^\perp &  -N\pa -M & -N\llts +M & +M\llts -N \\ \hline
\end{array}
$$

Conversely any output formula of linear logic is the translation of a Lambek formula and any input formula of linear logic is the negation of the translation of a Lambek formula. Let  $(\ldots)^\OC_\calTL$ denotes the inverse bijection of \ma{$+$}, from $\FO$
to
$\calTL$  and 
$(\ldots)^\IC_\calTL$ denotes the inverse bijection of \ma{$-$} from $\FI$
to  $\calTL$.
These two maps are inductively defined as follows:

$$
\begin{array}{|l||c|c|c|c|} \hline
F{\in}\FO 	&\alpha{\in}\calP 	& (G{\in}\FO) \llts (H{\in}\FO) &
(G{\in}\FI)\pa(H{\in}\FO) & (G{\in}\FO)\pa (H{\in}\FI)\\ \hline
F^\OC_\calTL 	&\alpha 	& G^\OC_\calTL \llts H^\OC_\calTL & 
G^\IC_\calTL\lto H^\OC_\calTL & G^\OC_\calTL
\lfrom H^\IC_\calTL\\ \hline
\hline
F{\in}\FI & \alpha^\perp {\in} \calP^\perp & (G{\in} \FI) \pa (H{\in}\FI) &
(G{\in}\FO)\llts (H{\in} \FI) & (G{\in} \FI)\llts (H{\in} \FO)\\ \hline
F^\IC_\calTL & \alpha 			& H^\IC_\calTL \llts 
G^\IC_\calTL & H^\IC_\calTL \lfrom G^\OC_\calTL
& H^\OC_\calTL\lto G^\IC_\calTL\\
\hline
\end{array}
$$

\subsection{Proof nets} 

A proof net is a graphical representation of a proof which identifies essentially similar proofs. A cut-free proof net has several conclusions, and it consists of 
\begin{itemize} 
\item 
 the subformula trees of its conclusions, that possibly stops on a sub formula which is not necessarily a propositional variable (axioms involving complex formulae simplify the learning process). 
 \item 
 a cyclic order on these sub formula trees 
 \item 
axioms that links two dual leaves $F$ and $F^\perp$ of these formula subtrees. 
\end{itemize} 

Such a structure can be represented by a sequence of terms --- admittedly easier to typeset  than a graph ---  with indices for axioms. 
Each index appears exactly twice, once on a formula $F$ (not necessarily a propositional variable) and one on $F^\perp$. 
Here are two proof nets with the same conclusions: 

\begin{exe} 
\ex \label{pneolso1} 
${\sss^\perp}^1\lts(\sss^2\pa {np^\perp}^3), np^3\lts ({\sss^\perp} \lts np)^7, ({np^\perp} \pa \sss)^7 \lts {\sss^\perp}^2, \sss^1$
\ex \label{pneolso2} 
${\sss^\perp}^1\lts(\sss^2\pa {np^\perp}^3), np^3\lts ({\sss^\perp}^4 \lts np^5), ({np^\perp}^5 \pa \sss^4) \lts {\sss^\perp}^2, \sss^1$
\end{exe}

The second  one is obtained from the first one by expansing the complex axiom $({\sss^\perp} \lts np)^7, ({np^\perp} \pa \sss)^7$ into two axioms: $({\sss^\perp}^4 \lts np^5), ({np^\perp}^5 \pa \sss^4)$. Complex axioms always can be expansed into atomic axioms --- this is known as $\eta$-expansion property. This is the reason why proof nets are often presented with atomic axioms. Nevertheless, we shall substitute propositional variables with complex formula during the learning process, and therefore we need to consider complex axioms as well --- see the processing of the example \ref{ex2RG} in section \ref{RG}.

No any such structure does correspond to a proof: 
\begin{definition}\label{pncond} 
A proof structure  with conclusions $C^1,I^1_1,\ldots,I^1_n$ is said to be a proof net of the Lambek calculus when it  enjoys the correctness criterion defined by the following properties: 
\begin{enumerate} 
\item \emph{Acyclic:} any cycle contains the two branches of a $\pa$ link  \label{acyclic} 
\item \emph{Intuitionistic:} exactly one conclusion is an output formula of $\FO$, all other conclusions are input formulae of $\FI$ \label{intui} 
\item \emph{Non commutative:} no two axioms cross each other \label{planarity} 
\item \emph{Without empty antecedent:} there is no sub proof net with a  single conclusion \label{emptysequence} 
\end{enumerate} 
\end{definition} 
The first point in this definition is not stated precisely but, given that we learn from correct structured sentences, we shall not need a precise definition. 
The reader interested in the details is referred to \cite[chapter 6]{MootRetore2012lcg}. 
Some articles on proof nets add to the criterion above a form of connectedness but it is not actually needed since this connectedness  is a consequence of the first two points see \cite{Guerrini2011} or \cite[section 6.4.8 pages 225--227]{MootRetore2012lcg}. 

\begin{definition}\label{pninduct}
Proof nets for the Lambek calculus can be defined inductively as follows (observe that they contain exactly one output conclusion): 
\begin{itemize} 
\item 
given an output formula $F$ an axiom $F, F^\perp$  is a proof net with two conclusions $F$ and $F^\perp$ --- we do no require that $F$ is a propositional variable. 
\item  
given a proof net $\pi^1$ with conclusions $O^1,I^1_1,\ldots,I^1_n$ 
and a proof net $\pi^2$ with output formula $O^2,I^2_1,\ldots,I^2_p$ 
where $O^1$ and $O^2$ are the output conclusions, one can add a $\lts$-link 
between a conclusion of one and a conclusion of the other, at least one of the two being an output conclusion. 
 We thus can obtain a proof net whose conclusions are: 
 \begin{itemize} 
 \item 
$O^1\lts I^2_k, I^2_{k+1},\ldots,I^2_{p},O^2,I^2_1,I^2_{k-1},I^1_1,\ldots,I^1_n$ --- $O^2$ being the output conclusion
\item 
$I^1_l\lts O^2, I^2_1,\ldots,I^2_p,I^1_{l+1},\ldots,I^1_{n},O^1,I^1_1,\ldots, I^1_{l-1},$ --- $O^1$ being the output conclusion
\item 
$O^1\lts O^2, I^2_1,\ldots,I^2_p, I^1_1,\ldots,I^1_n$ --- $O^1\lts O^2$ being the output conclusion. 
\end{itemize} 
\item 
given a proof net $\pi^1$ with conclusions $O^1,I^1_1,\ldots,I^1_n$  one can add a $\pa$ link between any two consecutive conclusions, thus obtaining a proof nets with conclusions:
\begin{itemize} 
\item 
$O^1,I^1_1,\ldots,I_i\pa I_{i+1},\ldots,I^1_n$ --- $O^1$ being the output conclusion
\item 
$O^1\pa I^1_1,I^1_2\ldots,I^1_n$ --- $O^1\pa I^1_1$ being the output conclusion
\item 
$I^1_n\pa O^1,I^1_1\ldots,I^1_{n-1}$ --- $O^1\pa I^1_1$ being the output conclusion
\end{itemize} 
\end{itemize} 
\end{definition}

A key result see e.g. \cite[Theorem 6.28]{MootRetore2012lcg}. 
 is that:  

\begin{theorem} 
The inductively defined proof nets of definition \ref{pninduct}, i.e. proofs,  exactly correspond 
to the proof nets defined as graphs enjoying the universal properties of the criterion \ref{pncond}
\end{theorem} 

A parse structure for a sentence $w^1,\ldots,w^p$ 
generated by a Lambek grammar $G$ defined by a lexicon $\lex_G$ 
is a proof net with conclusions 
$(c^n)^-,\ldots,(c^1)^-,\sss^+$ with $c^i\in \lex(w^i)$.
This replaces the definition of parse structure as normal natural deductions \cite{Tie99phd} which does not work well when the product is used \cite{AR07cie,Amblard07phd}.

Non-commutative proof-nets were first introduced by Roorda in his thesis 
\cite{Roo91}
 but he only showed the soundness of the criterion: the inductively defined proof nets enjoys a correctness criterion, namely the Danos Regnier criterion \cite{DR89} augmented with intuitionistic and non-commutative conditions. We provided the first proof of the converse in 
\cite{Ret96tal}. There are several equivalent notions of proof nets for cyclic multiplicative  intuitionistic linear logic that is  Lambek calculus. For instance the proof nets for non commutative linear logic by Abrusci and Ruet \cite{AR99}, restricted to the non commutative multiplicative connectives and to intuitionistic formulae also are a presentation of proof nets for Lambek calculus. It should be underlined, however that all these notions of proof nets do not properly handle cuts except the one by Melli\`es 
\cite{Mellies2004ribbon} --- see e.g. the discussions in 
\cite[chapters 11 and 18]{Girard2011blindspot}. 

\subsection{Structured sentences to learn from:  $\sss$ proof frames} 

An $\sss$ proof frame (\SPF) is simple a parse structure of a Lambek grammar i.e. a proof net whose formula names have been erased, except the $\sss$ on the output conclusion. Regarding axioms,  their positive and negative tips are also kept. 
Such a structure is the analogous of a functor argument structure for AB grammars or of a name free normal natural deduction for Lambek grammars used 
in \cite{BP90,Bus87b,BR01lll} 
and it can be defined inductively as we did in \ref{pninduct}, or by the conditions in definition \ref{pncond}.

\begin{definition}[$\sss$ proof frames, \SPF s]\label{frame} An \emph{\emph{$\sss$ proof frame (\SPF)}} is a normal proof net $\pi$ such that:
\begin{itemize}
\item 
The output of $\pi$ is labelled with the propositional constant $\sss$ --- which is necessarily the conclusion of an axiom, the input conclusion of this axiom being labelled $\sss^\perp$. 
\item 
The output conclusion of any other axiom in  $\pi$ is $O$ its input conclusion being $O^\perp=I$. 
\end{itemize} 

Given an $\sss$ proof net $\pi$ its associated $\sss$ proof frame $\pi_f$
is obtained by replacing in $\pi$ the output of any axiom by $O$ (and its dual by $I=O^\perp$) 
except the $\sss$ that is the output of $\pi$ itself which is left unchanged. 

A given Lambek grammar \emph{$G$ is said to generate an \SPF\ $\rho$} 
whenever  there exists  a proof net $\pi$ generated by $G$ such that $\rho=\pi^{IO}$. 
In such a case we write $\rho\in \SPF(G)$.  
\end{definition}

The \SPF s associated with the two proof nets \ref{pneolso1} and \ref{pneolso2} above are: 
\begin{exe}
\ex ${\sss^\perp}^1\lts(O^2\pa {I}^3), O^3\lts ({I}^4 \lts O^5), ({O}^5 \pa O^4) \lts {I}^2, \sss$
\ex ${\sss^\perp}^1\lts(O^2\pa {I}^3), O^3\lts I^7, O^7 \lts {I}^2, \sss$
\end{exe} 

Are proof frames related to partial proof nets or modules that first appeared in 
\cite{Gir86a}? They only \emph{partly} are!  Indeed, modules  are name-free parts of proof-nets and generalised such structures that can be combined to obtained real multiplicatives proof nets. Here we only consider name-free \emph{complete} proof nets, and as far as proof net surgery is concerned, we only replace some axioms with their $\eta$-expansion as explained in  
\cite[chapter 11]{Girard2011blindspot}.



\section{Unification, proof frames and  categorial grammars}
\label{uni} 

In this section, we shall briefly describe how categories, 
and more generally categorial grammars can be unified. 
Indeed, our learning algorithm makes a crucial use of category-unification,
and this kind of technique is quite common in grammatical inference \cite{JN99}. 

As said in paragraph \ref{categorial},  a categorial grammar is defined from a lexicon that maps every word $w$  to a finite set of categories $\lex_G(w)$. 
Categories are usually defined from a finite set $B$ of 
base categories that includes a special base category $\sss$.
Here we shall consider simultaneously many different categorial grammars 
since the learning hypothesis varies in the class of categorial grammars. 
In order to have a common set of base type for all categorial grammars, 
we shall consider an infinite set of base categories  $B$ whose members will be $\sss$ and infinitely many category variables 
denoted by $x$, $y$,
$x_1$, $x_2$, $\ldots$, $y_1$, $y_2$, $\ldots$ In other words,  
$B=\{\sss\}\cup V$, $\sss\not\in V$, $V$ being an infinite set of category variables. 
The categories arising from $B$ are defined as usual by 
$\calV \backus  \sss \naur  V \naur  \calV\lto\calV \naur  \calV\lfrom\calV \naur  \calV\otimes\calV$. 
This infinite set of base categories does not really modify categorial grammars: a given categorial grammar only makes use of finitely many base categories. Indeed, there are finitely many words each of them being associated with finitely many categories:  there are finitely many symbols in the lexicon hence only a finite number of base categories are used by a given categorial grammar. 
Choosing an infinite set of base categories  is rather important, as we shall substitute a category variable with a complex category using fresh variables, thus turning a categorial grammar into another one, and considering families of grammars over the same base categories.  

A \emph{substitution} $\sigma$ is a function from categories $\calV$ to categories $\calV$ that is generated by a mapping $\sigma_V$ of finitely many variables $x_{i_1},\cdots, x_{i_p}$ in $V$ to categories of $\calV$: 

$$
\begin{array}{rcl}
\sigma(\sss) & = &\sss \\
\mbox{given\ } x\in V,\quad  \sigma(x) & = & \left\{\begin{array}{ll} 
                           \sigma_V(x) & \mbox{if $x=x_{i_k}$ for some $k$} \\
                            x & \mbox{otherwise}
                           \end{array} \right. \\
\sigma(A\lto B) & = & \sigma(A)\lto\sigma(B) \\
\sigma(B\lfrom A) & = &\sigma(B)\lfrom\sigma(A)
\end{array}
$$ 

The substitution $\sigma$ is said to be a \emph{renaming} when $\sigma_V$ 
is a bijective mapping from $V$ to $V$ 
--- otherwise stated $\sigma_V$ is a permutation of the  $x_{i_1},\cdots, x_{i_p}$). 

As usual, substitutions may be extended to sets of categories 
by stipulating $\sigma(A)=\{\sigma(a)|a\in A\}$. Observe that $\sigma(A)$ can be a singleton while $A$ is not: 
$\{(a\lfrom (b\lto c)),(a\lfrom u)\}[u\mapsto(b\lto c]=\{a\lfrom (b\lto c)\}$. 
A substitution can also be applied to a categorial grammar: $\sigma(G)=G'$ with $\lex_{G'}(w)=\sigma(\lex_G(w))$
for any word $w$. Observe that a substation turns a $k$-valued (as defined in section  \ref{categorial}) categorial grammar into a $k'$-valued categorial grammar with $k'\leq k$, and possibly into a rigid (or $1$-valued) categorial grammar . 

A substitution $\sigma$ on Lambek categories (defined by  mapping finitely many category variables $x_i$ to Lambek categories $L_i$, $x_i\mapsto L_i$) 
clearly defines a substitution on linear formulae $\sigma^\ell$ (by $x_i\mapsto L_i^+$), which preserves the polarities $\sigma^\ell(F)$ is positive(respectively negative) if and only if $F$ is.  
Conversely, a substitution $\rho$ 
on linear formulae defined by mapping variables to positive linear formulae ($x_i\mapsto F_i$) defines a substitution on Lambek categories $\rho^L$ 
with the mapping $x_i\mapsto F^\OC_\calTL$.  One has: $\sigma(L)=(\sigma^\ell(L+))^\OC_\calTL$ and $\rho(F)=(\rho^L(F^\OC_\calTL))+$ if $F\in\FO$ 
and $\rho(F)=(\rho^L(F^\IC_\calTL))-$. Roughly speaking as far as we use only polarised linear formulae and substitution that preserve polarities, it does not make any difference to perform substitutions on linear formulae or on Lambek categories.

Substitution preserving polarities (or Lambek substitutions) 
can also be applied to proof nets: $\sigma(\pi)$ is obtained from $\pi$ by applying the substitution to any formula in $\pi$. 
A substitution turns an $\sss$ Lambek proof net into an $\sss$ Lambek proof net -- this is the reason why proof nets in this paper may contain axioms on complex formulae. 

\begin{proposition}\label{proofsubst} 
If \begin{itemize}
\item $\sigma$ is a substitution preserving polarities and 
\item $\pi$ a proof net generated by a Lambek grammar $G$,
\end{itemize} 
then
\begin{itemize} 
\item 
 $\sigma(\pi)$ is generated by $\sigma(G)$ and 
 \item 
 $\pi$ and $\sigma(\pi)$ have the same associated $\sss$ proof frame: $\sigma(\pi)_f=\pi_f$
 \end{itemize} 
\end{proposition} 

Two grammars $G_1$ and $G_2$  with their categories in $\calV$  are said to be \emph{equal} 
whenever 
there exists a renaming  $\nu$ such that $\nu(G_1)=G_2$. 

A substitution $\sigma$ is said to unify two categories $A,B$ if  one has $\sigma(A)=\sigma(B)$.
A substitution is said to unify a set of categories $T$ or to be a unifier for $T$  
if for all
categories $A,B$ in $T$ one has $\sigma(A)=\sigma(B)$ --- in other words, $\sigma(T)$ is a singleton. 

A substitution $\sigma$  is said to unify a categorial grammar $G$  or to be a unifier of $G$ whenever, for every word in the lexicon $\sigma$ unifies $\lex_G(w)$,
i.e. for any word $w$ in the lexicon $\lex_{\sigma(G)}(w)$ has a unique category --- in other words $\sigma(G)$ is rigid.

A unifier does not necessarily exists, but  when it does, there exists a  \emph{most general unifier (mgu)} that is a unifier $\sigma_u$ such for every unifier $\tau$ there exists a substitution $\sigma_\tau$ such that $\tau=\sigma_\tau \circ \sigma_u$. This most general unifier is unique up to renaming. 
This result also holds for unifiers that unify a set of categories and even for unifiers that unify a categorial grammar. 
\cite{Kan98}

\begin{figure} 
An algorithm for unifying two categories $C_1$ and $C_2$ 
may proceed  by managing  a finite multi-set $E$ of potential equations on terms, 
until it fails or reaches a set of equations whose left hand side are variables,  each of which appears in a unique such equation
--- a measure consisting in triple of integers ordered ensures that  this algorithm always stops. 
This set of equations $x_i=t_i$ defines a substitution by setting $\nu(x_i)=t_i$. 
Initially $E=\{C_1=C_2\}$. In the procedure below, upper case letters stand for categories, whatever they might be, $x$ for a variable, $*$ and $\diamond$ stand for binary connectives among $\lto,\lfrom,\otimes$. Equivalently, unifications  could be performed on linear formulae, as said in this article.  The most general unifier of $n$ categories can be performed by iterating  binary unification, the resulting most general unifier does not depend on  the way one proceeds. 

$$
\begin{array}{rcl} 
E\cup \{C{=}C\} &\yields& E \\ 
E\cup \{A_1{*} B_1 {=} A_2 {*} B_2 \}  &\yields& E\cup \{A_1{=}A_2 , B_1 {=} B_2 \} \\ 
E\cup \{C{=}x\} &\yields& E \cup \{x{=}C\}\\ 
\textrm{if }x\in Var(C) \quad   E\cup \{x{=}C\} &\yields& \fails \\ 
\textrm{if }x\not\in Var(C) \land x\in Var(E)   \quad E\cup \{x{=}C\} &\yields& E[x:{=}C] \cup \{x {=} C\} \\ 
\textrm{if }\diamond \neq {*} \quad E\cup \{A_1 {*} B_1 {=} A_2  \diamond  B_2 \}  &\yields&  \fails \\ 
E\cup \{\sss {=} A_2 {*} B_2 \}  &\yields&  \fails \\ 
E\cup \{A_1 {*} B_1 {=} \sss\}  &\yields& \fails \\ 
\end{array} 
$$
\caption{The unification algorithm for unifying two categories}
\label{unification} 
\end{figure}

\begin{definition} 
Let  $\pi$ be an $\sss$ proof net whose associated \SPF\ is 
$\pi_f$. 
If all the axioms in $\pi$ but the $\sss,\sss^\perp$ whose $\sss$  is $\pi$'s main output are $\alpha_i,\alpha_i^\perp$ with $\alpha_i\neq\alpha_j$ when $i\neq j$,
$\pi$ is said to be a \emph{most general labelling} of $\pi_f$. 
If $\pi_f$ is the associated \SPF\ of an $\sss$ proof net $\pi$ and $\pi_v$ one of the most general labelling of $\pi_f$, 
then $\pi_v$ is also said to be a most general labelling of $\pi$. 
The most general labelling of an $\sss$ proof net is unique up to renaming. 
\end{definition} 

We have the following obvious but important property: 

\begin{proposition} 
If $\pi_v$ is a most general labelling of an $\sss$ proof net $\pi$, then there exists a substitution $\sigma$ such that $\pi=\sigma(\pi_v)$. 
\end{proposition} 


\section{An RG-like algorithm for learning Lambek categorial grammars from proof frames}
\label{RG}

Assume that we wish to define a \emph{consistent}  learning function $\phi$ from positive examples for a class of categorial grammars (see definition \ref{learningprop}). 
Assume that $\phi$ already mapped  $e_1,\ldots,e_n$ to a grammar $G_n$ with $e_1,\ldots,e_n\subset \lang(G_n)$ ($\phi$ being consistent). If $e_{n+1}\in\lang(G_n)$ it is natural to define 
$\phi(e_1,\ldots,e_n,e_{n+1})=G_{n+1}$ as being $G_n$.
Otherwise, that is when $e_{n+1}\not\in\lang(G_n)$, 
there exists some word $w^k$ in the sentence $e_{n+1}$ 
such that no category of $\lex_{G_n}(w)$ is able to account for the behaviour of $w^k$ in the sentence $e_{n+1}$. 
A natural but misleading idea would be to say: if word $w^k$ needs category $c^k_{n+1}$ in  example $e_{n+1}$, 
let us add $c^k$  to $\lex_{G_n}(w^k)$ to define $\lex_{G_{n+1}}(w^k)$. Doing so for every occurrence of a problematic word in the sentence $e_{n+1}$, actually leads to  
 $e_1,\ldots,e_n,e_{n+1}\subset \lang(G_{n+1})$ and in the limit we
 should  obtain the smallest grammar $G_\infty$ such that $\forall i\ e_1,\ldots,e_i\in\lang{G_\infty}$. 
Doing so, there is little  hope to identify a language in the limit in Gold sense. Indeed, nothing guarantees that the process will stop, and a categorial grammar with infinitely many categories for some word is not even a grammar, that is a finite description of a possibly infinite language. 
Thus, an important guideline for learning categorial grammars is to bound the number of categories per word. That is the reason why we introduced in section \ref{categorial} the notion of \emph{$k$-valued} categorial grammars, with at most $k$ categories per word. We shall start by learning  \emph{rigid} ($1$-valued) Lambek categorial grammars with product (\LCGp) and this method extends to $k$-valued \LCGp. 

Our algorithm can be viewed as an extension to Lambek grammars with product of the RG algorithm (learning Rigid Grammars) introduced by 
Buszkowski and Penn in \cite{Bus87b,BP90} initially designed for rigid AB grammars.
A difference from their seminal work is that the data ones learns from 
were functor argument trees while here they are proof frames (or natural deduction frames  when the product is not used see section \ref{nd}). 
Proof frames may seem less natural than natural deductions, but we have two good reasons for using them:
\begin{itemize} 
\item Product is of interest for some grammatical constructions  as examples \ref{exetmoot} and \ref{expastpartmorrill} show while there is no fully satisfactory natural deduction for Lambek calculus with product. \cite{AR07cie,Amblard07phd}
\item Proof frames resemble dependency structures, since an axiom between the two conclusions corresponding to two words expresses a dependency between these two words. 
\end{itemize} 

To illustrate our learning algorithm we shall proceed with the three examples below, whose corresponding $\sss$ proof frames are given in figure \ref{3spn}. As their \SPF\ structures show, the 
middle one (\ref{ex2RG}) involves a positive product (the $I\pa I$ in the category of \ma{and}) 
and the last one (\ref{ex3RG}) involves an introduction rule (the $O\pa I$ in the category of \ma{that}). 

\begin{exe}
\ex \label{ex1RG} 
Sophie gave a kiss to Christian
\ex \label{ex2RG} 
Christian gave a book to Anne and a kiss to Sophie
\ex \label{ex3RG} 
Sophie liked a book that Christian liked. 
\end{exe} 


Usually,  in order to manipulate right handed sequents with conclusions only, 
proof nets reverse the order of the hypotheses which correspond to words, as explained in section \ref{cgpf} 
--- in some papers by  Glynn Morrill e.g. \cite{Mor98} the order is not reversed, but then the conclusions of the proof net, that are the linear formulae which are the dual of the Lambek categories are less visible. 
One solution that will make the supporters of either notation happy is to write the sentences vertically as  we do in figure \ref{3spn}.

\begin{figure} 
\exun
\hfill
\exdeux
\hfill 
\extrois
\vfill 
\caption{Three S proof frames:  three structured sentences for our learning algorithm.} 
\label{3spn} 
\end{figure}

\begin{definition}[RG like algorithm for \SPF s] 
Let $D=(\pi^k_f)_{1\leq k\leq n}$ be the $\sss$ proof frames associated with the examples $(e^k_f){1\leq k\leq n}$,
and let $(\pi^k)$ be most general labellings of the $(\pi^k_f)_{1\leq k\leq n}$. We can assume that they have no common category variables ---
this is possible because there are infinitely many category variables is infinite and because most general labellings are defined up to renaming. 
If example $e^k$ contains $n$ words $w^k_1,\ldots,w^k_n$ then $\pi^k$ has $n$ conclusions $(c^k_n)-,\ldots,(w^k_1)-,\sss$,
where all the $c^k_i$ are Lambek categories. 

Let $GF(D)$ be the (non necessarily rigid) grammar defined by the assignments $w_i^k:c_i^k$ --- observe that a for a given word $w$ there may exist several $i$ and $k$ such that $w=w_i^k$. 

Let $RG(D)$ be the rigid grammar defined as the most general unifier of the categories $\lex(w)$ for each word in the lexicon when such a most general unifier exists. 

We define $\phi(D)$ as $RG(D)$. When unification fails, the grammar is defined by $\lex(w)=\emptyset$ for those words whose categories do not unify.\footnote{There is an unimportant choice here: we could either say that $\phi$ is undefined in this case. In both cases $\phi$ does not seem to be consistent, that is to propose a grammar that actually generates the examples seen so far. However, as we shall see, in the convergence proof, when the algorithm is applied to a language in the class, categories of a given word always unify, and $\phi$ is a total and consistent learning function.\label{phipartial}}  
\end{definition} 

With the \SPF\ of our examples in \ref{3spn} yields the following type assignments where the variable $x_n$ corresponds to the axiom number $n$ in the examples, they are all different as expected
--- remember that  $\sss$ is not a category variable but a constant.

$$
\begin{array}{p{1.5cm}|l|l|} 
\textbf{word} & \mathbf{category\ (Lambek)} &  \mathbf{category^\perp\ (linear\ logic)}  \\ \hline 
and &  (( (x_{23}\otimes  x_{25})   \lto x_{22} ) ...\hspace*{2em}   & ((x_{28} \otimes  x_{27})\otimes  ...\\ 
 & \hfill    ...\lfrom (x_{28}\otimes  x_{27})) &  \hspace*{2em} ...(x_{22}\otimes  (x_{23}\otimes  x_{25}))) \\ \hline 
that & ((x_{34}\lto x_{33})\lfrom (x_{36}\lfrom x_{35})) & ((x_{36} \pa x_{35}^\perp) \otimes (x_{33}^\perp \otimes x_{34})) \\ \hline 
liked & (x_{31}\lto \sss)\lfrom x_{32} & x_{32} \otimes (\sss \otimes x_{31}) \\ \cline{2-3} 
 & (x_{37}\lto x_{36})\lfrom x_{35} & x_{35} \otimes (x_{36} \otimes x_{37}) \\ \hline 
gave & ((x_{11}\lto \sss) \lfrom (x_{13}\otimes x_{12})) & (x_{13} \otimes x_{12}) \otimes (\sss \otimes  x_{11}) \\  \cline{2-3} 
 & ((x_{21}\lto \sss) \lfrom x_{22}) & x_{22} \otimes  (\sss \otimes  x_{21}) \\  \hline 
to & x_{12}\lfrom x_{15} & x_{15} \otimes  x_{12}^\perp\\  \cline{2-3} 
 & x_{25}\lfrom x_{26} & x_{26} \otimes  x_{25}^\perp\\ \cline{2-3} 
 & x_{27}\lfrom x_{20} & x_{20} \otimes  x_{27}^\perp\\ \hline 
a & x_{13}\lfrom x_{14} &x_{14} \otimes  x_{13}^\perp\\ \cline{2-3} 
& x_{23}\lfrom x_{24} & x_{24} \otimes  x_{23}^\perp\\ \cline{2-3} 
 & x_{28}\lfrom x_{29} &  x_{29} \otimes  x_{28}^\perp\\ \cline{2-3} 
 & x_{32} \lfrom x_{33} & x_{33} \otimes  x_{32}^\perp\\ \hline 
 Anne & x_{26} & x_{26}^\perp\\ \hline 
Sophie & x_{11} & x_{11}^\perp\\ \cline{2-3} 
 & x_{20} & x_{20}^\perp\\ \cline{2-3} 
 & x_{31} & x_{31}^\perp\\ \hline  
Christian & x_{15} & x_{15}^\perp\\ \cline{2-3} 
 & x_{21} & x_{21}^\perp\\ \cline{2-3} 
 & x_{37} & x_{37}^\perp\\ \hline 
book & x_{24}& x_{24}^\perp  \\ \cline{2-3} 
 & x_{34}  & x_{34}^\perp\\ \hline  
kiss & x_{14} & x_{14}^\perp  \\ \cline{2-3} 
 & x_{29} & x_{29}^\perp \\ \hline 
\end{array} 
$$ 

Unifications either performed on Lambek categories $c^k_i$ 
or on the corresponding linear formulae (the $(c^k_i)-$ that appear in the second column) yield the following equations:

$$ 
\begin{array}[t]{l}
\mbox{liked}\\ \hline
x_{31}=x_{37}\\ 
x_{36}=\sss\\
x_{32}=x_{35}\\ \\ 
\mbox{gave}\\ \hline 
x_{11}=x_{21}\\ 
x_{22}=x_{13}\otimes x_{12}\\ \\ 
\mbox{to}\\ \hline 
x_{12}=x_{25}=x_{27}\\ 
x_{15}=x_{26}=x_{20}\\ 
\end{array}
\qquad \qquad \qquad 
\begin{array}[t]{l}
\mbox{a}\\ \hline 
x_{13}=x_{23}=x_{28}=x_{32}\\ 
x_{14}=x_{24}=x_{29}=x_{33}\\ \\ 
\mbox{Sophie}\\ \hline 
x_{11}=x_{20}=x_{31}\\ \\ 
\mbox{Christian} \\ \hline 
x_{15}=x_{21}=x_{37}\\ \\ 
\mbox{kiss}\\ \hline  
x_{14}=x_{29}\\ \\ 
\mbox{book}\\ \hline  
x_{24}=x_{34}\\ 
\end{array} 
$$

These unification equations can be solved by setting: 

$$
\begin{array}{l}
x_{36}=\sss 
\\ 
x_{22}=x_{13}\otimes x_{12}=np\otimes pp 
\\ 
x_{12}=x_{25}=x_{27}=pp  \qquad \mbox{prepositional phrase introduced by \ma{to}}
\\ 
x_{13}=x_{23}=x_{28}=x_{32} =x_{35}=np  \qquad  \mbox{noun phrase}
\\ 
x_{14}=x_{24}=x_{29}=x_{33}= x_{34}=cn  \qquad  \mbox{common noun}
\\ 
x_{11}=x_{20}=x_{31}=x_{15}=x_{21}=x_{37}= x_{15}=x_{26}=pn  \qquad  \mbox{proper name} 
\end{array} 
$$

The grammar can be unified into a rigid grammar $G_r$ , namely: 

$$
\begin{array}{p{1.5cm}|l|l|} 
\textbf{word} & \mathbf{category (Lambek)} &  \mathbf{category^\perp (linear logic)}  \\ \hline  
and &  (( (np\otimes  pp)   \lto (np \otimes pp )... \hspace*{1em}   & ((np\otimes pp)\otimes...  \\
 & \hfill    ...\lfrom (np \otimes  pp)) &  \hspace*{1em} ...((np\otimes pp)^\perp \otimes  (np\otimes pp))) \\ \hline  
that & ((n\lto n)\lfrom (\sss\lfrom np)) & ((\sss\pa np^\perp) \otimes (n^\perp \otimes n)) \\ \hline  
liked & (pp\lto \sss)\lfrom np & np \otimes (\sss \otimes pn) \\ \hline  
gave & (pp\lto \sss) \lfrom (pp\otimes np)) & (np \otimes pp) \otimes (\sss \otimes  pn) \\ \hline   
to & np \lfrom pn & pn \otimes  np^\perp\\ \hline  
a & np\lfrom cn & cn \otimes  pp^\perp\\ \hline  
Anne & pn & pn^\perp\\ \hline  
Sophie & pn & pn^\perp\\ \hline  
Christian & pn & pn^\perp\\ \hline  
book & cn & cn^\perp\\ \hline  
kiss & cn & cn^\perp\\ \hline  
\end{array} 
$$

As stated in proposition \ref{proofsubst},  one easily observes that the \SPF\ are indeed produced by the rigid grammar $G_r$. 

Earlier on, in the definition of an \SPF, we allowed non atomic axioms, and we can now precisely see why: 
the axiom $22$ could be instantiated by the single variable $x_{22}$
but, when performing unification,  it got finally instantiated with $x_{13}\otimes x_{12}$. 
Thus,  if we would have forced axioms to always be on propositional variables, 
the grammar $G_r$ would not have generated the \SPF\ of example $2$ 
but  the slightly different 
\SPF\ with  the axioms $x_{13}, x_{13}^\perp$ 
and  $x_{12}^\perp,x_{12}$ linked by an $\otimes$ link $x_{13}^\perp\otimes x_{12}$ 
and by a $\pa$ link $x_{12}^\perp\pa x_{13}^\perp$ in place of the axiom $22$.

\section{Convergence of the learning algorithm}
\label{conv}

This algorithm converges in the sense defined by Gold \cite{Gold67}, see definition \ref{goldcvg}. 
The first  proof of convergence of a learning algorithm for categorial grammars is the proof by Kanazawa  
 \cite{Kan94} of the convergence of the algorithm of Buszkowki and Penn \cite{BP90} for learning rigid basic  categorial grammars from functor argument  structures (name free natural deduction with $\lto$ and $\lfrom$ elimination rules only).
 Although we learn a different class of grammars from different structures, our proof is quite similar. 
 It follows \cite{Bon00} that is a simplification of  \cite{Kan98}. 
 
 
The proof of convergence makes use  of the following notions and notations: 
\begin{description} 
\item{$G\subset G'$} This reflexive relation between $G$ and $G'$
  holds whenever every lexical category assignment $a:T$ in $G$ is  in $G'$
  as well --- in particular when   $G'$ is rigid, so is $G$, and both grammars
  are identical. Note that this is just the normal subset relation for
  each of the words in the lexicon $G'$:
  $\lex_G(a) \subset \lex_{G'}(a)$ for every $a$ in the lexicon of
  $G'$, with $\lex_G(a)$ non-empty. Throughout the proof,  
  we shall  also use the subset relation symbol to signify inclusion of
  the generated \emph{languages}; the intended meaning of \ma{$\subset$} should always
  be clear from the context.
\item{\normalfont{\it size of a grammar}} The size of a grammar is
  simply the sum of the sizes of the occurrences of categories in the lexicon, where the size of a category is its number of occurrences of base categories (category variables or $\sss$). 
\item{$G\sqsubset G'$} This reflexive relation between $G$ and $G'$  holds when there exists a substitution $\sigma$ such that 
$\sigma(G)\subset G'$ which does not identify different categories of a given word, but this is 
always the case when the grammar is rigid. 
\item{$\SPF(G)$} As said earlier,  $\SPF(G)$ is the the set of $\sss$ proof frames  generated by a Lambek categorial grammar $G$. 
\item{$GF(D)$} Given a set $D$ of structured examples i.e. a set of $\sss$ proof frames,  the grammar $GF(D)$ is define as in the examples above: 
it is obtained by collecting the categories of each word in the various examples of $D$. 
\item{$RG(D)$} Given a set of \SPF s $D$, $RG(D)$ is 
the rigid grammar/lexicon obtained by applying the most general unifier, when it exists,  to $GF(D)$
---  in case the categories of a given word do not unify no category is assigned to this word, see footnote \ref{phipartial}. 
\end{description}




\begin{proposition}\label{finite} 
Given a grammar $G$, the number of 
grammars $H$ such that $H\sqsubset G$ is finite.
\end{proposition}


\begin{proof} There are only finitely many grammars which are included in $G$,
since $G$ is a finite set of assignments. 
Whenever $\sigma(H)=K$ for some substitution $\sigma$ 
the size of $H$ is smaller or equal to the size of $K$, 
and, up to renaming, there are only finitely many grammars smaller than a given grammar.

By definition, if $H\sqsubset G$ then there exist $K\subset G$ and a
substitution $\sigma$ such that $\sigma(H)=K$. Because there are only
finitely many $K$ such that $K\subset G$, and  for every $K$ there are
only finitely many $H$ for which there could exist a substitution $\sigma$ with $\sigma(H)=K$ (substitutions increase the category sizes)
we conclude that, up to renaming,  there are only finitely many $H$ such that $H\sqsubset G$.  \qed
\end{proof} 


From the definition of $\sqsubset$ and from proposition \ref{proofsubst} one immediately has: 

\begin{proposition}\label{substitution}  
If $G\sqsubset G'$ 
then $\SPF(G)\subset \SPF(G')$.
\end{proposition}


\begin{proposition}\label{b->a}
If $GF(D)\sqsubset G$ then $D\subset \SPF(G)$. 
\end{proposition} 

\begin{proof} 
By construction of $GF(D)$, we have $D\subset \SPF(GF(D))$. In addition, because of
proposition~\ref{substitution}, we have $\SPF(GF(D))\subset \SPF(G)$. \qed
\end{proof} 

\begin{proposition}\label{genall}
If $RG(D)$ exists then $D\subset \SPF(RG(D))$.
\end{proposition}

\begin{proof}
By definition $RG(D)=\sigma_u(GF(D))$ where $\sigma_u$ is the most general unifier of all the categories of each word. So we have $GF(D)\sqsubset RG(D)$,
and applying proposition~\ref{b->a} with $G=RG(D)$ we obtain 
$D\subset \SPF(RG(D))$. \qed
\end{proof} 

\begin{proposition}\label{a->b} 
If $D\subset \SPF(G)$ then $GF(D)\sqsubset G$.
\end{proposition}


\begin{proof} 
By construction of $GF(D)$, 
each category variable $x$ labels at most one axiom of  at most one  
\SPF\ of $D$. 
According to the hypothesis $D\subset \SPF(G)$, every \SPF\ $e_i$  in $D$ is the \SPF\  
associated with an $\sss$ proof net $\pi_i$ generated by $G$, and let us chose one such $\pi_i$ in case there are several of them. For every category variable $x$ labelling  the positive tip of an axiom $ax_i^j$ 
in some of the $e_i$ 
we can define a substitution by 
 $\sigma(x)=T$ where $T$ is the category that labels the positive tip of the same axiom $ax_i^j$ in $\pi_i$: 
indeed $x$ occurs once,
and such a substitution is well defined.   When this substitution is applied to $GF(D)$ it 
yields a grammar which only contains assignments from $G$ --- by applying the substitution to the whole 
\SPF, it remains a well-categorised \SPF, and in particular the formulae  on
the conclusions corresponding to the words, that are the dual of the Lambek categories in the lexicon,  must coincide. 
\footnote{One can alternatively proceeds with positive linear formulae $F$ as subsection \ref{formulae} shows.} Hence we find a substitution such that $GF(D)\subset G$.   
\qed
\end{proof}

\begin{proposition}\label{existence} 
When $D\subset \SPF(G)$ with $G$ a rigid grammar, 
the grammar $RG(D)$ exists and $RG(D)\sqsubset G$.
\end{proposition}

\begin{proof} 
By proposition~\ref{a->b} we have $GF(D)\sqsubset G$, 
so  there exists a substitution $\sigma$ such that $\sigma(GF(D))\subset G$.

As $G$ is rigid, $\sigma$ unifies all the categories of each word. 
Hence there exists a unifier of all the categories of each word, and $RG(D)$ exists. 

$RG(D)$ is defined as the application of most general unifier
$\sigma_u$ to $GF(D)$.  By the definition of a most general unifier\footnote{Unifiers and most general unifiers 
work as usual even though we unify \emph{sets} of categories, see section \ref{unification}.}, 
there exists a substitution $\tau$ such that $\sigma=\tau\circ\sigma_u$. 

Hence $\tau(RG(D))=\tau(\sigma_u(GF(D)))=\sigma(GF(D))\subset G$;
\newline thus $\tau(RG(D))\subset G$, hence $RG(D)\sqsubset G$. \qed
\end{proof}

\begin{proposition}\label{increasing} 
If $D\subset D'\subset \SPF(G)$ with $G$ a rigid grammar 
then $RG(D)\sqsubset RG(D')\sqsubset G$. 
\end{proposition}

\begin{proof} 
Because of proposition~\ref{existence} both $RG(D)$ and $RG(D')$ exist. 
We have $D\subset D'$ and $D'\subset \SPF(RG(D'))$,
so $D\subset \SPF(RG(D'))$;
hence, by proposition~\ref{existence} applied to $D$ and $G=RG(D')$ (a rigid grammar) 
we have $RG(D)\sqsubset RG(D')$. \qed
\end{proof} 

\begin{theorem}
The algorithm RG for learning rigid Lambek grammars converges in the sense
of Gold. 
\end{theorem}

\newcommand\ent{\omega}
\begin{proof}
Let  $(D_i)_{i\in\nat}$ be an increasing sequence of \emph{sets} of examples in
$\SPF(G)$ enumerating $\SPF(G)$, in other words  $\cup_{i\in\ent} D_i=\SPF(G)$: 

$$D_1 \subset D_2 \subset \cdots D_i \subset D_{i+1} \cdots \subset \SPF(G)$$ 

Because of proposition~\ref{existence} 
for every $i\in\ent$ the rigid grammar $RG(D_i)$ exists and because of
proposition~\ref{increasing} the rigid  grammars $RG(D_i)$ define a  $\sqsubset$-increasing sequence 
of grammars 
which by proposition~\ref{existence} 
is bounded by $G$: 

$$RG(D_1) \sqsubset RG(D_2) \sqsubset \cdots RG(D_i) \sqsubset RG(D_{i+1}) \cdots \sqsubset G$$

As they are only finitely many grammars  $H\sqsubset G$ (proposition~\ref{finite}) 
this sequence $RG(D_i)$ is stationary
after a certain rank:  there exists an integer $N$ such that for all $n\geq N$ $RG(D_n)=RG(D_N)$. 

Let us show that the langue generated by $RG(D_N)$ is the one to be learnt, 
i.e. let us prove that $\SPF(RG(D_N))=\SPF(G)$ by proving the two inclusions: 
\begin{enumerate} 
\item Firstly, let us prove that $ \SPF(RG(D_N))\supset \SPF(G)$
Let $\pi_f$ be an \SPF\ in $\SPF(G)$. 
Since $\cup_{i\in\ent} D_i=\SPF(G)$ there exists a $p$ such that $\pi_f\in \SPF(D_p)$. 
\begin{itemize} 
\item 
If $p<N$, because $D_p\subset D_N$, $\pi_f\in D_N$,
and by proposition~\ref{genall} $\pi_f\in \SPF(RG(D_N))$. 
\item 
If $p\geq N$, we have $RG(D_p)=RG(D_N)$ since the sequence of grammars is stationary after $N$. 
By proposition~\ref{genall} we have $D_p\subset \SPF(RG(D_p))$ 
hence $\pi_f\in \SPF(RG(D_N))=\SPF(RG(D_p))$. 
\end{itemize} 
In all cases, $\pi_f\in \SPF(RG(D_N))$. 
\item Let us finally prove that $\SPF(RG(D_N))\subset \SPF(G)$: 
Since $RG(D_N)\sqsubset G$, by proposition~\ref{substitution} we have 
$\SPF(RG(D_N))\subset \SPF(G)$  \qed
\end{enumerate} 
\end{proof}

This precisely shows that the algorithm proposed in section 
\ref{RG} 
converges in the sense of Gold's definition (\ref{goldcvg}).

\section{Learning product free Lambek grammars from  natural deduction frames}  
\label{nd} 

The reader may well find that the structure of the positive examples that we  learn from,  sorts of proofnets 
are too sophisticated structures to learn from.  He could think that our learning process is a drastic simplification of the similar algorithms that use functor argument structures, i.e. name free natural deductions. 

Let us first see that normal natural deductions are quite a sensible structure to learn Lambek grammars from. 
Tiede \cite{Tie99phd} observed that natural deductions  in the Lambek calculus (be they normal or not) are plain trees, defined by two unary operators ($\lto$ and $\lfrom$ introduction rules) and two binary operators ($\lto$ and $\lfrom$ elimination rules), from formulae as leaves (hypotheses, cancelled or free). As opposed to the intuitionistic case, there  is no need to specify which hypothesis is cancelled by the introduction rules, as they may be inferred inductively: a $\lto$ (respectively $\lfrom$) introduction rule cancels the left-most (respectively right-most) free hypothesis. 
He also observed that \emph{normal} natural deductions should be considered as the proper parse structures, since otherwise any possible syntactic structure (a binary tree) is possible. 
Therefore is is natural to learn Lambek grammars from normal natural deduction frames --- natural deductions from which category names have been erased but the final $\sss$. Indeed, $\sss$ natural deduction frames are to Lambek categorial grammars what the functor-argument (FA) structures are to AB categorial grammars --- these FA structures are the standard structures used for learning AB grammars by Buskowski, Penn and Kanazawa \cite{BP90,Kan98}.

The purpose of this section is to exhibit a one to one correspondence between cut-free proof nets of the product free Lambek calculus and normal natural deductions, thus justifying the use of proof frames for learning Lambek grammars. When there is no product, proof frames are the same as natural deduction frames that we initially used in \cite{BR01lll}. They generalise the standard FA structures, and when the product is used, natural deduction become quite tricky \cite{AR07cie,Amblard07phd} and there are the only structures one can think about.

The correspondence between on one hand natural deduction or the isomorphic $\lambda$-terms and on the other hand, proof nets, can be traced back to \cite{Ret87} (for second order lambda calculus) but the 
the closest result is the one for linear $\lambda$-calculus \cite{GR96}.

\subsection{Proofnets and natural deduction: climbing principal branches} 
\label{pn&nd}

As said in section \ref{categorial}, the  formulae of product free Lambek calculus are defined
 by:  
$$\calC \backus  \sss \naur  B \naur  \calC\lto \calC \naur  \calC\lfrom \calC$$
Hence their linear counterpart are a strict subset of the polarised linear formulae of subsection \ref{formulae}: 
$$\left\{
\begin{array}{lcccccccc}
\FOH & \backus  &\calP & \naur   & (\FIH\pa\FOH) & \naur  & (\FOH\pa \FIH)\\
\FIH & \backus  & \calP^\perp & \naur  & (\FOH\llts\FIH) & \naur  & 
(\FIH\llts \FOH) \\
\end{array}\right.
$$ 

Let us call these formulae the \emph{heterogeneous}  polarised formulae, which are either positive or negative formulae.
In these heterogeneous formulae 
the connectives $\pa$ and $\otimes$ may only apply to a pair formulae with opposite polarity. 
The translation from Lambek categories to linear formulae and vice versa from subsection \ref{formulae} apply to them as well.

One may  think that a proof net corresponds to a sequent calculus proof which itself  corresponds to a natural deduction: as shown in our book \cite{MootRetore2012lcg}, this is correct, as far as one does not care about \emph{cuts} --- which are problematic in non commutative calculi, see e.g.\cite{Mellies2004ribbon}.  
As it is well known in the case of intuitionnistic logic, cut-free and normal are different notions \cite{Zuc74}, and proof net are closer to sequent calculus in some respects. If one translate inductively, rule by rule, a natural deduction into a sequent calculus or into a proof net, the  elimination rule from $A$ and $A\lto B$ yields a cut on the $A\lto B$ formula, written $A^\perp\pa B$ in linear logic. We shall see how this can be avoided.

\begin{figure}
\begin{center}
\begin{prooftree} 
	\[ \begin{array}{c}	
	\mbox{\footnotesize\it  this rule requires at least two free hyp.} \\ 
	\\ 
	   A\ \ \mbox{leftmost free hyp.}\\ 
	   \ldots [A]\ldots \ldots 
	   \end{array} 
	\leadsto 
	B 
	\using 
	\]
	\justifies 
	A\lto B
	\using \mbox{$\lto_{i}$\ binding $A$} 
\end{prooftree} 
\hspace{2cm}
\begin{prooftree} 
	\[\Delta 
	\leadsto 
	A
	\] 
	\[\Gamma
	\leadsto 
	A\lto B 
	\] 
	\justifies 
	B
	\using\lto_{e}
\end{prooftree} 
\end{center}
\begin{center}
\begin{prooftree} 
	\[ \begin{array}{c}
	\mbox{\footnotesize\it this rule requires at least two free hyp.} \\ 
	\\ 
	   A\ \mbox{rightmost free hyp.}\\ 
	   \ldots \ldots [A] \ldots 
	   \end{array} 
	\leadsto 
	B 
	\using 
	\]
	\justifies 
	B\lfrom A
	\using \mbox{$\lfrom_{i}$\ binding $A$} 
\end{prooftree} 
\hspace{2cm} 
\begin{prooftree} 
	\[\Gamma
	\leadsto 
	B\lfrom A 
	\] 
	\[\Delta 
	\leadsto 
	A
	\] 
	\justifies 
	B
	\using\lfrom_{e}
\end{prooftree} 
\end{center} 
\caption{Natural deduction rule for product free Lambek calculus}
\label{natded} 
\end{figure}

\subsubsection{From normal natural deductions to cut-free proof nets} 

Let us briefly  recall some basic facts on natural deduction for the product free Lambek calculus, from our book   \cite[section 2.6 pages 33-39]{MootRetore2012lcg}. 
In particular we shall need the following notation. Given a formula $C$, and a sequence of length $p$ of pairs
consisting of a letter $\varepsilon_i$ (where $\varepsilon_i \in \{ l, r\}$) and a formula $G_i$ 
we denote by $$C[(\varepsilon_1,G_1),\ldots,(\varepsilon_p,G_p)]$$ the  formula defined as follows:  
\begin{description}
\item[if $p=0$] $C[]=C$ 
\item[if $\varepsilon_p=l$]
  $C[(\varepsilon_1,G_{1}),\ldots,(\varepsilon_{p-1},G_{p-1}),(\varepsilon_p,G_p)]=\newline   
  G_p\lto C[(\varepsilon_{1},G_{1}),\ldots,(\varepsilon_{p-1},G_{p-1})]$
\item[if $\varepsilon_p=r$]  $C[(\varepsilon_1,G_1),\ldots ,(\varepsilon_{p-1},G_{p-1}),(\varepsilon_p,G_p)]=
\newline C[(\varepsilon_{1},G_{1}),\ldots,(\varepsilon_{p-1},G_{p-1})]\lfrom G_p$ 
\end{description}

An important property of normal natural deductions is that whenever the last rule is an elimination rule, there is a principal branch 
leading from the conclusion to a free hypothesis \cite[proposition 2.10 page 35]{MootRetore2012lcg} 
When a rule $\lto_e$ (resp. $\lfrom_e$) is applied between a right premise $A\lto X$ (resp. a left premise 
$X\lfrom A$) and a formula $A$ as its left (resp. right) premise, the premise  $A\lto X$ (resp. a left premise 
$X\lfrom A$) is  said to be the \textit{principal} premise. In a proof ending with an elimination rule, a \emph{principal branch} is a path 
from the root $C=X_0$ to a leaf $C[(\varepsilon_1,G_1),\ldots,(\varepsilon_p,G_p)]=X_p$ such that one has 
$X_i=C[(\varepsilon_1,G_1),\ldots,(\varepsilon_i,G_i)]$ and  also $X_{i+1}=C[(\varepsilon_1,G_1),\ldots,(\varepsilon_{i+1},G_{i+1})]$
and $X_i$ is the conclusion of an  elimination rule, $\lto_e$ if $\varepsilon_{i+1}=l$  and $\lfrom_e$ 
if $\varepsilon_{i+1}=r$, 
with principal premise $X_{i+1}$ and $G_{i+1}$ as the other premise.

 Let $d$ be a normal natural deduction with conclusion $C$ and hypotheses $H_1,\ldots,H_n$. 
The deduction $d$ is inductively turned into a cut-free proof net with conclusions $H_n-,\ldots,H_1-,C+$ as follows (we only consider $\lto$ because $\lfrom$ is symmetrical). 
 \begin{itemize} 
 \item 
 If $d$ is just an hypothesis $A$ which is at the same time its conclusion the corresponding proof net is the axiom $A,A^\perp$. 
 \item 
 If $d$ ends with a $\lto$ intro, from $A,H_1,\ldots,H_n\seq B$ to $H_1,\ldots,H_n\seq A\lto B$, by induction hypothesis we have a proof net with conclusions $(H_n)-,\ldots,(H_1)-,A-,B+$. 
 The heterogeneous $\pa$ rule applies since $B+$ is heterogeneous positive and $A-$ heterogeneous negative. 
 A $\pa$ rule yields a proof net with conclusions $(H_n)-,\ldots,(H_1)-,A-\pa B+$, and $A-\pa B+$ is precisely $(A\lto B)+$
 \item 
 The only interesting case is when $d$ ends with an elimination rule, say $\lto_e$.
 In this case there is a principal branch, say with hypothesis $C[(\varepsilon_1,G_1),\ldots,(\varepsilon_p,G_p)]$ which is applied to $G_i$'s. Let us call $\Gamma_i=H_i^1,\ldots,H_i^{k_{i}}$ the  hypotheses of $G_i$, and let $d_i$ be the proof of $G_i$ from $\Gamma_i$. 
 By induction hypothesis we have a proof net $\pi_i$ with conclusions $(\Gamma_i)-,(G_i)+$.
Let us define the proof net $\pi^k$  of conclusion $C^k-=C[(\varepsilon_1,G_1),\ldots,(\varepsilon_k,G_k)]-$, 
$\Gamma_i$ for $i\leq k$ and $C+$ by:
\begin{itemize}
\item 
 if $k=0$ then it is an axiom $C^\perp, C$  (consistent with the translation of an axiom) 
 \item 
 otherwise $\pi^{k+1}$ is obtained  by a times rule between the conclusions 
  $C^k-$ of 
 $\pi^k$ and $G_{k+1}+$ of $\pi_{k+1}$ 
When $\varepsilon_i=r$ then the conclusion chose the conclusion of this link to  
$G_{k+1}+\otimes C^{k}-$ that is $C^k-\lfrom G_{k+1}+=C^{k+1}-$ and when 
 $\varepsilon_i=l$ the conclusion is $C^{k}-\otimes G_{k+1}+$ that is $G_{k+1}+\lto C^{k}-=C^{k+1}-$.
 hence, in any case the conclusions of $\pi^{k+1}$ are $C^{k+1}+$ $C+$ and the $\Gamma_i$ for $i\leq k+1$. 
\end{itemize} 
The translation of $d$ is simply $\pi^p$, which has the proper conclusions. 
 \end{itemize} 
 
 As the translation does not introduce any cut-rule, the result is a cut-free proof net. 
 
 \subsubsection{From cut-free proof nets to normal natural deductions} 
 
 There is an algorithm that performs the reverse translation, presented for multiplicative linear logic  and linear lambda terms in \cite{GR96}. It strongly relies on the correctness criterion, 
which makes sure that everything happens as indicated during the algorithm and that it terminates. 
This algorithm always points at a formula in the proof net, and draws paths in the proof net.  Going \emph{up} means going to an immediate sub formula, and going \emph{down} means considering the immediate super formula.  The algorithms label the proof net nodes with Lambek lambda terms that are natural deductions written as terms, and the natural deduction that translates the proof net is the Lambek lambda term labelling the output of the proof net. 
 
 \begin{enumerate}
 \item \label{init}  Enter the proof net by its unique output conclusion. 
 \item \label{up}
Go up until you reach an axiom. Because of the polarities, during this upwards path, you only meet $\pa$-links, which  correspond to the introduction rules  $\lambda_r x_i^{T_i}$ or $\lambda_l x_i^{T_i}$, the $T_i$s being the  input formulae (the hypotheses that are cancelled).  
Such formulae are labelled with distinct variables $x_i$. 
 \item \label{down}
 Use the axiom link and go down with the input polarity. Hence you only meet $\otimes$ links (*) until you reach a conclusion or a $\pa$ link. In both cases, this formula is the type of the head-variable of the normal Lambek $\lambda$-term. If it is the premise of a 
 $\pa$-link, then it is necessarily a $\pa$ link on the path of step \ref{up} (because of the correctness criterion). In this case, the head variable (the hypothesis of the principal branch) 
is  bound by the corresponding $\lambda_r$ or $\lambda_l$ of the previous step \ref{up}. 
  Otherwise it  the head variable   is free.  
 \item \label{iteration}
 The output formulae that were left unlabelled when going down are the output premises of the 
 $\otimes$ links (*) that we met at step \ref{down}. To label them, one  goes up from theses output formulae, applying again  step \ref{up}. 
 \end{enumerate} 
 
 The  $\lambda$-term that labels the output conclusion is normal: only variables are applied to some arguments during the translation. It is easily read as a normal natural deduction. 

\subsection{Learning product free Lambek grammars from natural deduction} 

We have defined a bijective correspondence between \emph{cut free}  product free  proof nets and \emph{normal} product free  natural deduction.  Therefore we also have a correspondence between $\sss$ proof frames and name free natural deduction  whose conclusion is $\sss$. 

Hence, if one wishes to,  it is possible to learn product free Lambek grammars from natural deduction without names but the final $\sss$, as we did in \cite{BR01lll}. 
 Such structures are simply the generalisation to Lambek calculus of the FA structures that are commonly used for basic categorial grammars  by \cite{BP90,Kan98}. 

\section{Conclusion and possible extensions} 
\label{concl}

A criticism that can be addressed to our learning algorithm  is that the rigidity condition on Lambek grammars is too restrictive. 
One can say, as in \cite{Kan98} that $k$-valued grammars can be learned by doing all the possible unifications that lead to less than $k$ categories. Every successful unifications yielding a grammar with less than $k$ categories  should be kept, because in a later step it is quite possible that one works while the others do not: hence this approach is computationally intractable. An alternative is to use 
a precise part-of-speech tagger and to consider one word with different categories as several distinct words. 
This looks more accurate and has been carried out effectively,  with the help of some statistical techniques. \cite{MS2012lacl,moot10semi}

The principal weakness of identification in the limit is that too much structure is required on the input examples.  
Ideally, one would like to learn directly from strings, but in the case of Lambek grammars it has been shown to be impossible in \cite{FL02coling}.  One may think that it could be possible 
to try every possible structure on sentences as strings of words as done in  \cite{Kan98} for basic categorial grammars. 
Unfortunately, in the case of Lambek grammars, with or without product, this cannot be done. Indeed,  there can be infinitely many structures corresponding to a sentence, because a cancelled hypothesis does not have to be anchored in one the finitely many words of the sentence. Hence we ought to learn from structured sentences, as we did. 

From the point of view of first language acquisition we know that some structure is available, 
but it is unlikely that the structured sentences are  the proof frames of the present article.  The real structure available to the learner includes prosodic and semantic informations, and no one knows how to formalise these structures in order 
to simulate the natural data used during the actual language learning process. 
From a computational linguistic perspective, our result is not as restrictive as it may seem. 
Indeed, there exist tools that annotate corpora, and one may implement other tools that turn standard annotations into the annotations we need. These shallow and efficient processes may lead to structures from which one can infer the proper structure for an algorithm like the one we presented in this paper. In the case of proof nets or frames, as observed long ago,  axioms express the consumption of the valencies. This  is the reason why, apart from the structure of the formulae, the structure of the proof frames is not so different from dependency annotations and such annotations can be used to infer categorial structures 
as done by Moot and Sandillon-Rezer  \cite{MS2012lacl,moot10semi}. However,  the automatic acquisition of wide-coverage grammars for natural language processing applications, certainly requires a combination of machine learning techniques and of  identification in the limit \`a la Gold, although up to now there are not so many such works.  

Grammatical formalisms that can be represented in Lambek grammars can also be learnt like we did in this paper. 
For instance, a categorial version of Stabler's minimalist grammars \cite{Sta96} can be learnt that way as the attempts by Fulop or by us show \cite{fulop2004logic,BR01lll}
This should be even better with the so-called Categorial Minimalist grammars of Lecomte, Amblard and us 
\cite{Amblard07phd,ALR2010Lambek}

\bibliographystyle{splncssrt}
\bibliography{bigbiblio} 

\end{document}